\documentclass[11pt,onecolumn,a4paper]{article}

\usepackage{amsmath}
\usepackage{amsfonts}
\usepackage{amssymb}
\usepackage{amsthm}
\usepackage{cite}

\usepackage[latin1]{inputenc}
\usepackage{graphicx,color}
\usepackage{url}
\usepackage{hyperref}
\usepackage{latexsym,enumerate}
\usepackage{booktabs}
\usepackage{framed}
\usepackage{subfigure}
\usepackage{multirow} 
\usepackage{color}
\usepackage{colortbl}
\definecolor{gray1}{rgb}{0.8,0.8,0.8}
\definecolor{gray2}{rgb}{0.95,0.95,0.95}

\newtheorem{Th}{Theorem}[section]
\newtheorem{Lem}[Th]{Lemma}

\newcommand{\RE}{\,{\rm Re}}
\newcommand{\IM}{\,{\rm Im}}
\newcommand{\ARG}{\,{\rm Arg}}
\newcommand{\ETAL}{{\it et al.}}
\newcommand{\IE}{{\it i.e.,}}

\newcommand{\figref}[1]{Fig. \ref{#1}}
\renewcommand{\eqref}[1]{Eq. (\ref{#1})}

\newcommand{\tableref}[1]{Table \ref{#1}}

\def\anglepoint{anchor point}
\def\anglepoints{anchor points}

\def\AnglePoints{Anchor Points}

\newcommand\commentout[1]{}

\begin{document}

\title{Perfect Fingerprint Orientation Fields \\by Locally Adaptive Global Models}

%\author{Carsten Gottschlich, Benjamin Tams, and Stephan Huckemann}

\author{Carsten Gottschlich\thanks{Institute for Mathematical Stochastics,
University of G\"ottingen,
Goldschmidtstr. 7, 37077 G\"ottingen, Germany.
Email: gottschlich@math.uni-goettingen.de},%
\ Benjamin Tams\thanks{Institute for Mathematical Stochastics,
University of G\"ottingen,
Goldschmidtstr. 7, 37077 G\"ottingen, Germany.
Email: btams@math.uni-goettingen.de}, and%
\ Stephan Huckemann\thanks{Felix-Bernstein-Institute 
for Mathematical Statistics in the Biosciences,
University of G\"ottingen,
Goldschmidtstr. 7, 37077 G\"ottingen, Germany.
Email: huckeman@math.uni-goettingen.de}}

\maketitle

%\tableofcontents
%\clearpage

\begin{abstract}

Fingerprint recognition is widely used for verification and identification in many commercial, governmental and forensic applications. The orientation field (OF) plays an important role at various processing stages in fingerprint recognition systems. OFs are used for image enhancement, fingerprint alignment, for fingerprint liveness detection, fingerprint alteration detection and fingerprint matching. In this paper, a novel approach is presented to globally model an OF combined with locally adaptive methods. We show that this model adapts perfectly to the 'true OF' in the limit. This perfect OF is described by a small number of parameters with straightforward geometric interpretation. Applications are manifold: Quick expert marking of very poor quality (for instance latent) OFs, high fidelity low parameter OF compression and a direct road to ground truth OFs markings for large databases, say. In this contribution we describe an algorithm to perfectly estimate OF  parameters automatically or semi-automatically, depending on image quality, and we establish the main underlying claim of high fidelity low parameter OF compression.

\end{abstract}

\section*{Keywords}

Fingerprint recognition, orientation field estimation, orientation field models,
orientation field compression, orientation field marking,
latent fingerprint recognition, fingerprint image enhancement,
fingerprint matching

\section{Introduction} 

The orientation field (OF) is a crucial ingredient 
for most fingerprint recognition systems \cite{HandbookFingerprintRecognition2009}.
An OF is an image (or matrix) which encodes at each pixel 
of the fingerprint foreground \cite{ThaiHuckemannGottschlich2016}
the orientation $o(x,y) \in [0, 180[$ degrees (or $o(x,y) \in [0, \pi[$ in radians)
of a tangent to the ridge and valley flow at location $(x, y)$.
Modelling and estimating OFs
is a fundamental task for automatic processing of fingerprints.

We introduce a locally adaptive global model called the 
\textit{extended quadratic differential} (XQD) model.
We show that XQDs can model the OF of a fingerprint perfectly in the limit.
The major advantage of the XQD model lies in its small number of parameters, each of which has
a simple and obvious geometric meaning.

OFs have many important areas of application
at various stages of processing fingerprints.
In the following part of this section, we discuss some of the most relevant of these applications.
The rest of this manuscript is organized as follows. 
In Section \ref{sec:ofliterature}, we review related work from the literature
for estimating, modelling and marking OFs of fingerprints.
In Section \ref{sec:XQDM}, we describe the novel XQD model.
In Section \ref{convergence}, we prove that the XQD model adapts perfectly 
to the OF of a real fingerprint in the limit.
In Section \ref{sec:results}, we present practical results 
for compressing real fingerprint OFs by XQD models.
We conclude in Section \ref{sec:conclusion} with a discussion 
and we point out topics for future work.

\subsection{Image Enhancement} 

Most systems for fingerprint verification and identification 
are based on minutiae templates.
Automatic extraction of minutiae from fingerprints
can be a very challenging task for images of low and very low quality.
Quality loss of fingerprint images acquired on optical scanners can occur
if a finger is too dry, or too wet, or contains scars.
Putting a finger with too much or too little pressure on sensor 
can have similar negative effects on the image quality.

Poor image quality can cause a minutiae extraction module
to miss some true minutiae and to introduce some spurious minutiae.
The goal of fingerprint image enhancement is to avoid these two types of errors
by improving the image quality prior to minutiae extraction. 
The most effective approach for fingerprint image enhancement is contextual filtering
and the most important type of local context is the OF.

For example, oriented diffusion filtering \cite{GottschlichSchoenlieb2012} 
uses only the OF to perform anisotropic smoothing along the ridge and valley flow. 
Curved regions \cite{Gottschlich2012} are computed based on the OF 
and first, they are used for estimating the local ridge frequency 
and subsequently, OF and ridge frequency estimates are joint inputs
for curved Gabor filtering \cite{Gottschlich2012}.
These two methods estimate the local context and perform filtering in the spatial domain.
Alternatively, methods for contextual filtering of fingerprints can also operate in the Fourier domain ,
see e.g. methods proposed by Chikkerur \ETAL \cite{ChikkerurCartwrightGovindaraju2007}
and by Bartunek \ETAL \cite{BartunekNilssonSallbergClaesson2013}.
A hybrid approach with processing steps in the spatial and Fourier domain 
has been suggested by Ghafoor \ETAL \cite{GhafoorTajAhmadJafri2014}.

\subsection{Liveness Detection} 

Software-based fingerprint liveness detection strives to classify 
an input fingerprint as belonging to one of two classes: 
An image of an alive, real finger or an image of a fake or spoof finger 
made from artificial material like wood glue, gelatine or silicone.
Developing countermeasures against spoof attacks is a very active research area.
Two methods apply the OF to compute invariant descriptors:
For histograms of invariant gradients (HIG) \cite{GottschlichMarascoYangCukic2014},
the gradient direction at each pixel is normalized relative to the local orientation.
Convolution comparison patterns (CCP) \cite{Gottschlich2016} 
are obtained from small image patches. To that end, rotation-invariant patches are computed 
by locally rotating each window according to the local orientation at that pixel.

\subsection{Alteration Detection} 

Fingerprint alteration is another type of presentation attack \cite{SousedikBusch2014}
in which the attacker has the goal of avoiding identification 
(e.g. attempting to not being found in a watchlist search during border crossing
or not being identified in a forensic investigation).
Altered Fingerprints often have a disturbed OF. 
Therefore it is not surprising that in recent comparisons 
of features for alteration detection \cite{GottschlichMikaelyanOlsenBigunBusch2015},
some of the most effective features are related to the OF. 
In a nutshell, DOFTS \cite{BigunMikaelyan2015} and OFA \cite{YoonFengJain2012} 
are based on the difference between an estimated OF and a smoother version of it,
COH \cite{GottschlichSchoenlieb2012} relies on the coherence of gradients 
and SPDA \cite{EllingsgaardSousedikBusch2014} makes use of the fact 
that alterations tend to introduce additional singularities into a fingerprint.

\subsection{Matching} 

Orientation descriptors have been proposed by Tico and Kuosmanen \cite{TicoKuosmanen2003}
for computing the similarity between two minutiae from two templates.
These local similarities are aggregated into a global score
which summarizes the similarity between both templates.

Improvements of fingerprint recognition performance have been observed 
by using differences between two aligned OFs for score revaluation \cite{Gottschlich2010PhD}:
First, two minutiae templates are matched 
and the output is a global similarity score and a minutiae pairing.
Second, the corresponding OFs are aligned, based on the paired minutiae
and, the similarity between the OFs is evaluated. 
On the one hand, if both OFs fit well together, the aligned OFs confirm the minutiae pairing and the global score is increased.
On the other hand, if major discrepancies between the aligned OFs are observed, 
this is considered as an indication of a potential impostor recognition attempt and the score is decreased accordingly.

\subsection{Alignment}

OFs are used for fingerprint alignment (also known as registration),
i.e. finding a global rotation and translation of one OF with respect to other 
which is obtained by optimizing a cost function \cite{YagerAmin2004}.
Krish \ETAL \cite{KrishFierrezRamosOrtegagarciaBigun2015} considered in their work
the alignment of partial fingerprints from fingermarks to 'full' fingerprints,
so called 'rolled' fingerprints 
which are acquired with the help of e.g. a	police officer who rolls the finger of subject 
to capture the full surface from nail to nail.
Similar to above described score revaluation \cite{Gottschlich2010PhD}, 
they found that OF alignment improves the recognition performance, 
in their case, the rank-1 identification rate.

Tams \cite{Tams2013} studied the problem of absolute pre-alignment
of a single fingerprint in the context of fingerprint-based cryptosystems \cite{Tams2012PhD,Tams2016}
and suggested an OF based method.

\subsection{Classification and Indexing} 

The goal of classification and indexing
is to speed up fingerprint identification (1 to N comparisons, 
where N can be in the magnitude of millions for forensic databases, 
see Chapter 5 in \cite{HandbookFingerprintRecognition2009}).
For example, the class tented arch is observed in about 3\% of all fingerprints \cite{HandbookFingerprintRecognition2009}.
Hence, if a query fingerprint belongs to the class tented arch, 
then the search space can reduced by 97\%. 
The majority of approaches for classification and indexing relies on OFs, e.g. 
Cappelli \ETAL \cite{CappelliLuminiMaioMaltoni1999}
proposed a method for fingerprint classification by directional image partitioning.
A recent survey by Galar \ETAL \cite{GalarClassificationI}
lists 128 references and most of them use the OF (or its singular points) 
for classification.

\subsection{Synthetic Fingerprint Generation} 

The generation of artificial fingerprint images has the advantage 
that it is possible to create arbitrarily large databases for research purposes 
e.g. of a million or a billion fingerprints at virtually no cost and without legal constraints.
Methods for producing synthetic fingerprints include 
\cite{CappelliErolMaioMaltoni2000,AraqueEtAl2002,KueckenChampod2013,ImdahlHuckemannGottschlich2015}.
A detailed discussion of approaches for constructing and reconstructing fingerprint images 
can be found in \cite{GottschlichHuckemann2014}.

All methods have in common that they require an OF for the image creation process.
The methods by Cappelli \textit{et al.} \cite{CappelliErolMaioMaltoni2000}
and by Araque  \textit{et al.} \cite{AraqueEtAl2002}
rely on the global OF model by Vizcaya and Gerhardt \cite{VizcayaGerhardt1996}.
In contrast, the realistic fingerprint creator (RFC) \cite{ImdahlHuckemannGottschlich2015}
uses OF estimations by a combination of gradient-based and line sensor methods \cite{GottschlichMihailescuMunk2009}
from a database of real fingerprints.

\subsection{Separation of Overlapping Fingerprints} 

During a forensic investigation, it can occur that traces at a crime scene are detected
where two or more latent fingerprints overlap on a surface. The task is to separate these fingerprints,
so that the separated single fingerprints can individually be utilized for identification.
Several research groups have addressed this problem in their work,
and the key to the solution are in each work the OFs,
see e.g. \cite{QianSchottZhengDittmann2014,FengShiZhou2012,ZhaoJain2012}.

A different forensic problem studied by Hildebrandt and Dittmann \cite{HildebrandtDittmann2015}
is latent fingerprint forgery detection. 
For this application one may well compute rotationally invariant features 
such as HIG \cite{GottschlichMarascoYangCukic2014} or CCP \cite{Gottschlich2016} 
by taking the orientation flow at the latent fingerprint into account.

\section{Estimation, Modelling and Marking of Orientation Fields} \label{sec:ofliterature}

Considering the importance of the OF, 
it is no surprise that a large body of literature is treating 
the topic of automatic OF estimation. 
A classic approach is to estimate the OF 
by some form of averaging (squared) image gradients (computed e.g. using the Sobel filter)
or symmetry features, see e.g. \cite{BigunGranlund1987,Bigun1988,BazenGerez2002,Bigun2006}.
However, this works only for good quality images. 
Further approaches include complex 2D energy operators \cite{Larkin2005}.
For dealing also with medium and low-quality fingerprint images, 
the line sensor method \cite{GottschlichMihailescuMunk2009,GottschlichMihailescuMunk2009ISPA}
was developed which recently has been adapted to detect the oriented filaments
in microscopy images \cite{EltznerWollnikGottschlichHuckemannRehfeldt2015}.
A dictionary based method \cite{FengZhouJain2013}
has been proposed for estimating the OF in latent fingerprints.
Many additional references can be found in \cite{HuckemannHotzMunk2008,GottschlichMihailescuMunk2009,TurroniMaltoniCappelliMaio2011}
and Chapter 3 of \cite{HandbookFingerprintRecognition2009}.

\subsection{Global Models}

The zero-pole model has been introduced by Sherlock and Monro in 1993 \cite{SherlockMonro1993}.
The flow fields generated by the zero-pole model resemble in some generality 
OFs of fingerprints, 
however they deviate significantly from OFs of a real finger.
Vizcaya and Gerhardt improved the simple zero-pole model in 1996 \cite{VizcayaGerhardt1996}
by suggesting an additional nonlinear bending scheme to better fit the OF generated by their model
to real OFs.
A global model based on quadratic differentials (QD) has been proposed
by Huckemann \textit{et al.} in 2008 \cite{HuckemannHotzMunk2008}.
The zero-pole model is a special case of this more general model
which has five geometrically interpretable parameters.
The QD model better fits real OFs especially for the fingerprint of the type arch.
Further global models include the work by Ram \ETAL \cite{RamBischofBirchbauer2010}
who apply Legendre polynomials for OF modelling.

\subsection{Manually Marking of Orientation Fields}

There are two main motivations for manually marking information in fingerprints.
The first is the creation of ground truth information which can be used 
for evaluating the performance of human experts and algorithms 
regarding the estimation or extraction of the target information.
And the second use case is the labeling for (semi)automated retrival of information
such as the foreground region, singular points, orientations 
or minutiae for fingerprint images which are too difficult for automatic processing 
by current state-of-the-art automatic fingerprint identification systems (AFIS) software. 
Forensic examiners mark such information in latent fingerprints 
to identify suspects in criminal investigations.
Advancements in latent fingerprint recognition have the goal 
of minimizing the time and effort required by human experts for successful identifications.

\subsubsection{Evaluating Orientation Field Estimation Performance}

In order to compare the performance of different OF estimation methods,
1782 orientations at specific locations in various fingerprints in have been manually marked by Gottschlich \ETAL \ in \cite{GottschlichMihailescuMunk2009} 
with a focus on low-quality regions affected by noise.
Cappelli \ETAL \cite{CappelliMaioMaltoni2009} addressed the problem of enhancing very low-quality fingerprints 
and suggested to manually mark the OFs used for contextual filtering.
They proposed to mark local orientations, compute the Delaunay triangulation
and interpolate the orientation at unmarked pixel locations inside a triangle from 
the marked orientations at the three vertices of triangle.
A disadvantage of this approach is that a large number of small triangles is required to approximate
the true orientation in highly curved regions around singular points. 
Cappelli \ETAL \cite{CappelliMaltoniTurroni2010} 
and 
Turroni \ETAL \cite{TurroniMaltoniCappelliMaio2011} created a ground truth benchmark called FOE 
following the same marking strategy (10 good and 50 bad quality prints). 
They compared the OF estimation performance of several algorithms from literature on this benchmark.
The FOE benchmark has recently also been used 
for evaluating the performance of methods 
which reconstruct OFs from minutiae \cite{OehlmannHuckemannGottschlich2015}.
In our work, OF compression results using the FOE benchmark are reported in Section \ref{sec:results}. 
We note that XQD models can be used as an alternate interpolation method 
not suffering from the need of large numbers of support points at high curvature near singularities.

\subsubsection{Latent Fingerprint Recognition}

Latent fingerprint recognition is still considered to be a difficult problem.
The level of noise for some fingermarks from crime scenes can be high  
and depending the surface from which fingermarks are lifted (or directly photographed),
a complex background can make the recognition task far more difficult
in comparison with the processing fingerprints captured by a fingerprint sensor. 
Typical first steps, among them fingerprint segmentation 
\cite{ThaiHuckemannGottschlich2016,ThaiGottschlich2016G3PD} 
and OF estimation, are challenging.
Recently, a novel image decomposition technique called DG3PD has been introduced 
which can better cope with these challenging images, see Figures 9 and 10 in \cite{ThaiGottschlich2016DG3PD}.

The goal of fully automatic latent fingerprint identification has not yet been achieved.
Even state-of-the-art commercial latent identification software fails for a considerable amount of images
and information still has to be manually marked by forensic experts in these cases.
E.g. in a work by Yoon \ETAL \cite{YoonFengJain2010}, information about the region-of-interest (ROI), 
the location of singular points and the orientation at some sparse locations is still assumed to be manually marked.
In the light of these problems, a subordinate target is to minimize the time and effort 
required by a human experts and the XQD model proposed in our work approach can be instrumental in achieving this.

\subsection{Compression of Orientation Fields and Fingerprint Images}

Forensic and governmental databases can contain millions of fingerprint images. 
Storing large volumes of data efficiently is a key issue 
which can be addressed by image compression \cite{Salomon2007Compression}.
Th\"{a}rn\r{a} \ETAL \cite{ThaernaNilssonBigun2003}
suggested to utilize the OF for improving lossless fingerprint image compression.
More specifically, they suggest to increase redundancy by scanning pixels along the orientation,
instead of standard procedures like horizontal (row by row) or vertical scanning of images.
Larkin and Fletcher \cite{LarkinFletcher2007} proposed a method for lossy fingerprint image compression
by decomposing an image into four elemental images which are highly compressible.
One of these four images, called the continuous phase, can be converted into an OF 
and vice versa.
Both approaches by Th\"{a}rn\r{a} \ETAL \cite{ThaernaNilssonBigun2003}
and by Larkin and Fletcher \cite{LarkinFletcher2007} 
can profit from improvements of the OF compression by our XQD models. 
If in an application e.g. by a law enforcement agency, 
fingerprint images and their minutiae templates are stored together, 
an straightforward idea would be to reconstruct the OF from the minutiae template. 
However, a recent evaluation of OF reconstruction methods \cite{OehlmannHuckemannGottschlich2015}
showed that all existing methods have weaknesses,
and especially in proximity to singular points,
all methods tend to be very inaccurate.
In an analogy, minutiae templates can viewed as a form of lossy fingerprint image compressions.
A survey of methods for reconstructing fingerprint images from minutiae templates 
can be found in \cite{GottschlichHuckemann2014}.
Recently, Shao \textit{et al.} \cite{ShaoWuYongGuoLiu2014}
studied fingerprint image compression by use of dictionaries of fingerprint image patches.
An additional discussion of texture image compression can be found in Section 7.3 in \cite{ThaiGottschlich2016DG3PD}.
The efficiency of XQD models for OF compression will be detailed in Sections \ref{sec:ofcompression} and \ref{sec:results}.

\section{XQD Models} \label{sec:XQDM}

Our methods for manually marking and automatically compressing fingerprint OFs 
are based on the \emph{quadratic differential} (QD) model of Huckemann \ETAL \cite{HuckemannHotzMunk2008}. 
Consequently, we shall outline that model first. 

\subsection{The Quadratic Differential Model} 

The basis of this model is given by a model for the arch type fingerprint. Adding given singular point coordinates, \IE{} cores and delta, 
this can be generalized to model the other fingerprint types. 

The OF of an arch type fingerprint is roughly controlled by two parameters: 
Given a Cartesian coordinate system $(x,y)$ in complex coordinates $z=x+\sqrt{-1}y$, 
the OF is linked to the following complex function
\begin{equation} \label{eq:arch}
P(z) = (z^2-R^2)^{2}
\end{equation}
for $\IM(z)>0$ and, otherwise, $P(z)=1$, as follows. The orientation angle at the coordinate $(x,y)$ can be obtained by 
\begin{equation} \label{eq:QDMAngles}
A(x,y)=0.5\cdot\ARG(P(x+\lambda\sqrt{-1}\cdot y))
\end{equation}
with the main branch of the argument of a complex number taking values in $[-\pi,\pi)$. 
The parameter $R>0$ controls the coordinates of two singularities $(-R,0)$ 
and $(R,0)$ (2nd order zeroes of $P$) along the abscissa and $\lambda>0$ 
is a factor controlling vertical stretching. In \figref{fig:arch} a reasonable fit of the QD model 
to an arch type fingerprint is visualized where $R$, $\lambda$, 
and the rotation and translation of the coordinate system have been adjusted.

\begin{figure}[!ht]
\begin{center}
\subfigure[]{\includegraphics[width=0.475\textwidth]{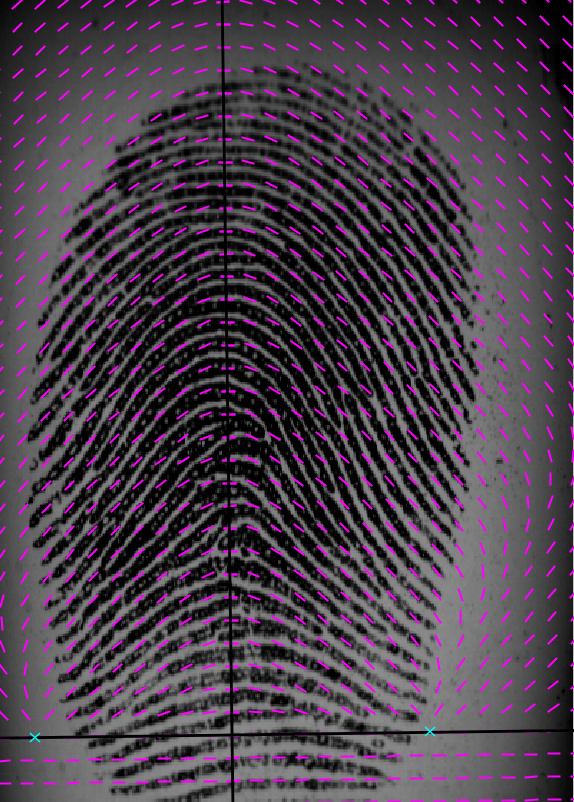}\label{fig:arch}}\hspace{0.01\textwidth}\subfigure[]{\includegraphics[width=0.475\textwidth]{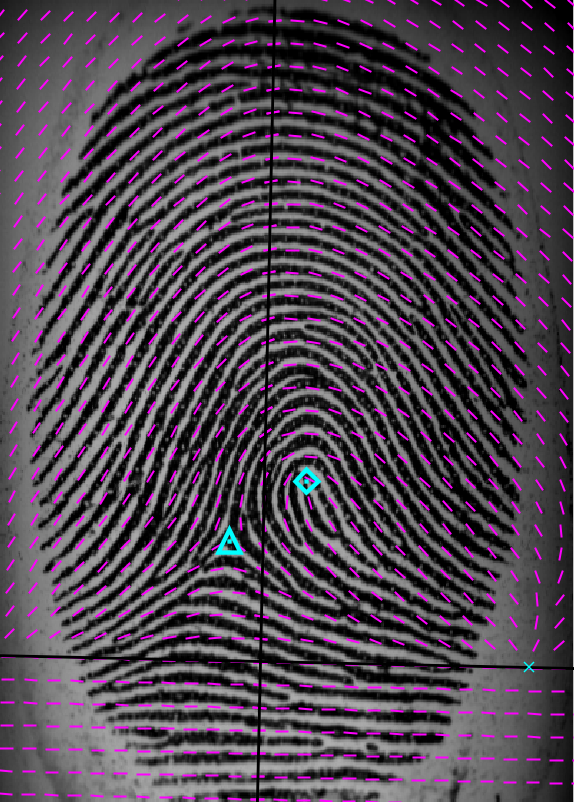}\label{fig:loop}}
\caption{Orientation fields modeled by QD models for an arch type (a) and a loop type (b) fingerprint.}
\label{fig:QDM}
\end{center}
\end{figure}

Fingerprints of other types (such as loops, double loops, and whorls) contain an equal number of deltas and cores --- where a fingerprint cannot contain more than two deltas/cores; note that a whorl can be considered as a double loop in which the two cores agree or are of small distance. The following formula extends \eqref{eq:arch} to also model an OF of a loop type fingerprint of which core and delta coordinates are encoded by the complex $\gamma$ and $\delta$, respectively:
\begin{equation}\label{eq:loop}
P_{\gamma,\delta}(z)=P(z)\cdot\frac{(z-\gamma)(z-\overline{\gamma})}{(z-\delta)(z-\overline{\delta})}
\end{equation}
for $\IM(z)>0$ and, otherwise, $P_{\gamma,\delta}(z)=1$. Here $\overline{z} = \RE(z) - \sqrt{-1}\IM(z)$ denotes the complex conjugate of $z$. An OF of a loop type fingerprint modeled by the QD model is visualized in \figref{fig:loop}. Similarly, a double loop with complex core coordinates $\gamma_1,\gamma_2$ and complex delta coordinates, $\delta_1,\delta_2$ is modeled by the following:
\begin{equation}\label{eq:DoubleLoop}
P_{\gamma_1,\gamma_1,\delta_1,\delta_2}(z)=P(z)\cdot\frac{(z-\gamma_1)(z-\overline{\gamma_1})\cdot(z-\gamma_2)(z-\overline{\gamma_2})}{(z-\delta_1)(z-\overline{\delta_1})\cdot(z-\delta_2)(z-\overline{\delta_2})}
\end{equation}
for $\IM(z)>0$ and, otherwise, $P_{\gamma_1,\gamma_2,\delta_1,\delta_2}(z)=1$. For both models,  \eqref{eq:loop} and  \eqref{eq:DoubleLoop}, orientation angles are computed via \eqref{eq:QDMAngles}. 

For a more comprehensive treatment of the QD model, we refer the reader to \cite{HuckemannHotzMunk2008} and the literature therein on geometric function theory; there the inverse $Q=P^{-1}$ of $P$ is considered giving the 
\emph{quadratic differential} (QD)
$$ \frac{dz^2}{P(z)} = Q(z)\,dz^2\,,$$
the solution curves of $Q(z)\,dz^2>0$ having the orientations from  \eqref{eq:QDMAngles}.
Then, in particular the ``zeroes'' of $P$ are in fact poles of the QD and  ``poles'' of $P$ are zeroes of the QD.

\subsection{Extended Quadratic Differential Model}
As can be seen in \figref{fig:QDM}, the QD model can be used to quite well approximate the general ridge flow using few parameters only.  However, the reader quickly recognizes areas in which the model significantly deviates from the evident ground-truth ridge flow which is an unavoidable effect due to the fact that $P$ in the QD model has only few degrees of freedom. Consequently, we need to change or extend the model. In this paper, we propose to attach a variable number of local correction points to which we refer as \emph{\anglepoints{}} thereby obtaining an \emph{extended quadratic differential} (XQD) model. With these points the local OF modeled by a QD can be corrected to better match with the ridge flow of a fingerprint. 

\subsubsection*{\AnglePoints}

\begin{figure}[!ht]
\begin{center}
\subfigure[]{\includegraphics[width=0.49\textwidth]{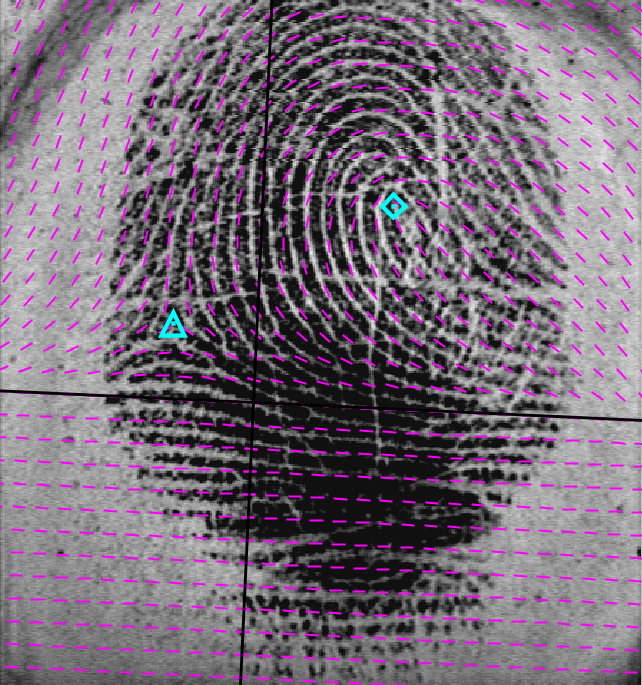}}\hspace{0.0033\textwidth}\subfigure[]{\includegraphics[width=0.49\textwidth]{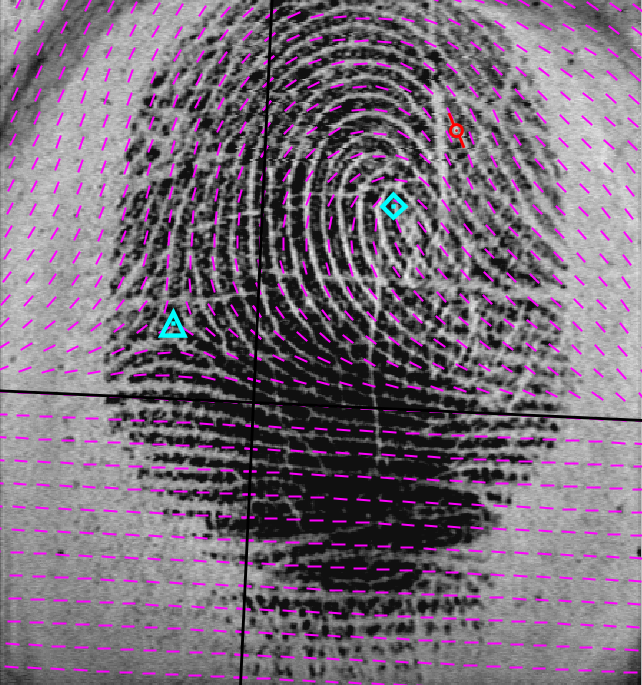}}\\\subfigure[]{\includegraphics[width=0.49\textwidth]{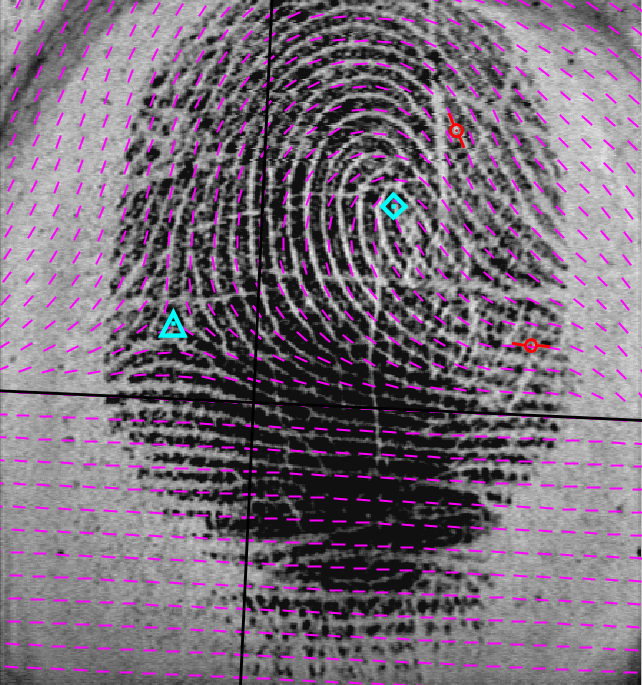}}\hspace{0.0033\textwidth}\subfigure[]{\includegraphics[width=0.49\textwidth]{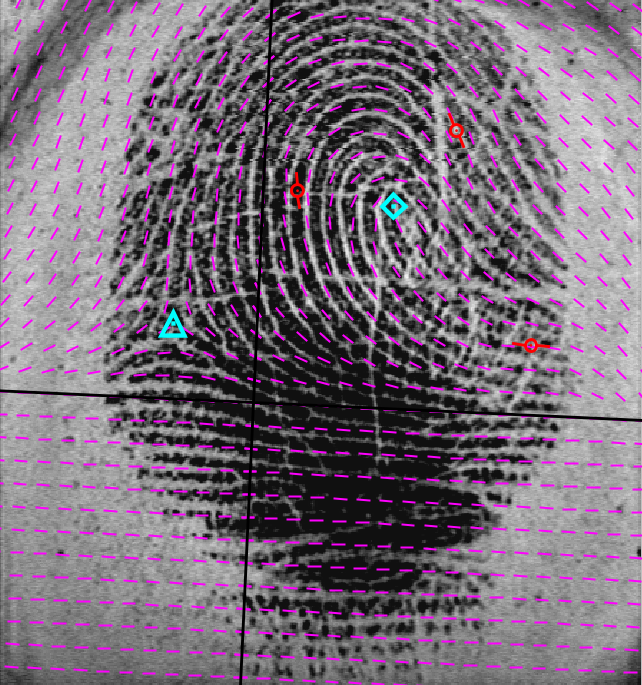}}
\caption{Extending the QD model by a variable number of \anglepoints}
\label{fig:XQDM}
\end{center}
\end{figure}

An \emph{\anglepoint} is a $5$-tuple $(a,b,\theta,\sigma_1,\sigma_2)$ where $(a,b)$ is a two-dimensional coordinate, $\theta$ an orientation angle, and $\sigma_1$ and $\sigma_2$ are two postive numbers. More precisely, $(a,b)$ denotes a coordinate at which the orientation given by a QD model $P(z)$ is to be corrected;  $\theta \in [-\pi/2,\pi/2)$ denotes the orientation angle of the true field at $(a,b)$ which is to become the new orientation angle there; finally, $\sigma_1$ and $\sigma_2$ control how significantly the orientation correction influences the neighboring orientations around $(a,b)$. Even more specifically, given the orientation angles $A(x,y)$ of a QD (see \eqref{eq:QDMAngles}), a true orientation $\theta$ at an \anglepoint{} $p=(a,b,\theta,\sigma_1,\sigma_2)$ the new orientation angles at any coordinate $(x,y)$ is computed as
\begin{equation}
A(x,y;p) = \begin{cases}
A(x,y) + C(x,y;p) - \pi&\text{if }A(x,y) + C(x,y;p)\geq\pi/2\\
A(x,y) + C(x,y;p) + \pi&\text{if }A(x,y) + C(x,y;p)<-\pi/2\\
A(x,y) + C(x,y;p)&\text{otherwise}
\end{cases}
\end{equation}
where $C(x,y;p)\in(-\pi/2,\pi/2]$ denotes a correction angle. The correction angle is defined as
\begin{equation} \label{eq:CorrectionAngle0}
C(x,y;p)=w(x,y;p)\cdot\begin{cases}
\theta-A(a,b)+\pi&\text{if }\theta-A(a,b)<-\pi/2\\
\theta-A(a,b)-\pi&\text{if }\theta-A(a,b)\geq \pi/2\\
\theta-A(a,b)&\text{otherwise}
\end{cases}
\end{equation}
where $w(x,y;p)$ denotes a function that assumes the value $1$ at $(x,y) = (x_p,y_p)$ and decays quickly to zero away from it. Here
$(x_p,y_p)$ is the coordinate $(x,y)$ represented w.r.t. a coordinate system defined by $(a,b)$ (origin) and $\theta$ (rotation); specifically,
\begin{align}
\begin{split}
x_p&=\cos(\theta)\cdot(x-a)+\sin(\theta)\cdot(y-b)+x\\
y_p&=\sin(\theta)\cdot(x-a)-\cos(\theta)\cdot(y-b)+y.
\end{split}
\end{align}
For example $w(x,y;p)$ can be a tent function as in \eqref{tent-fcnt:eq} with $\sigma_1=r=\sigma_2$. To obtain a higher degree of smoothness, in the applications we use  
the two-dimensional Gaussian
\begin{equation}
w(x,y;p)=\exp(-1/2\cdot(x-x_p)^2/\sigma_1^2+(y-y_p)^2/\sigma_2^2)\,.
\end{equation}

\subsubsection*{Multiple \AnglePoints}
Similarly, given the OF $A(x,y)$ of a QD, the correction angle at $(x,y)$ can be defined recursively from a multiple number of \anglepoints{} $p_1,...,p_n$ as
\begin{footnotesize}
\begin{align}
\begin{split}
C(x,y;p_1,...,p_{n})=
\begin{cases}
C(x,y;p_{1},...,p_{n-1})+C(x,y;p_{n})+\pi&\text{if }C(x,y;p_{1},...,p_{n-1})+C(x,y;p_{n})<-\pi/2\\
C(x,y;p_{1},...,p_{n-1})+C(x,y;p_{n})-\pi&\text{if }C(x,y;p_{1},...,p_{n-1})+C(x,y;p_{n})\geq\pi/2\\
C(x,y;p_{1},...,p_{n-1})+C(x,y;p_{n})&\text{otherwise}
\end{cases}
\end{split}
\end{align}
\end{footnotesize}
for $n>1$ and as in \eqref{eq:CorrectionAngle0} for $n=1$. This yields our final XQD model
\begin{footnotesize}
\begin{equation} \label{eq:XQDM}
A(x,y;p_1,...,p_{n})=\begin{cases}
A(x,y)+C(x,y;p_1,...,p_n)-\pi&\text{if }A(x,y)+C(x,y;p_{1},...,p_{n})\geq\pi/2\\
A(x,y)+C(x,y;p_1,...,p_n)+\pi&\text{if }A(x,y)+C(x,y;p_{1},...,p_{n})<-\pi/2\\
A(x,y)+C(x,y;p_1,...,p_n)&\text{otherwise}
\end{cases}.
\end{equation}
\end{footnotesize}
In \figref{fig:XQDM} the effect of correcting a QD model's OF using an increasing number of \anglepoints{} is visualized.

\subsection{Manually Marking of Orientation Fields}
One important application of our XQD model is to manually mark semi-automatically a fingerprint's OF by an expert. From the many choices of orders of tasks, by preliminary experiments, we found the following strategy useful, the steps of which are visualized in \figref{fig:mark}.
\begin{enumerate}
\item\label{step:MarkSPs} Manually mark the position of all cores and deltas of the fingerprint (\figref{fig:MarkSPs}).
\item\label{step:MarkInitialOF} Manually mark an initial OF (possibly at sparse locations only, see \figref{fig:MarkInitialOF}).
\item\label{step:AdjustQDM} Adjust the QD model to the initial OF by minimizing a suitable objective function, given by \eqref{eq:ObjectiveFunction}, say (\figref{fig:AdjustQDM}).
\item\label{step:InsertAPs} Successively insert \anglepoints{} to the XQD model further minimizing the objective function (\figref{fig:InsertAPs}).
\item The final XQD model agrees, within a preselected error bound, say, with the manually marked OF. This and other stopping strategies are discussed and illustrated in Section \ref{sec:results}.
\end{enumerate}

\begin{figure}[!ht]
\begin{center}
\subfigure[mark all cores and deltas]{\includegraphics[width=0.45\textwidth]{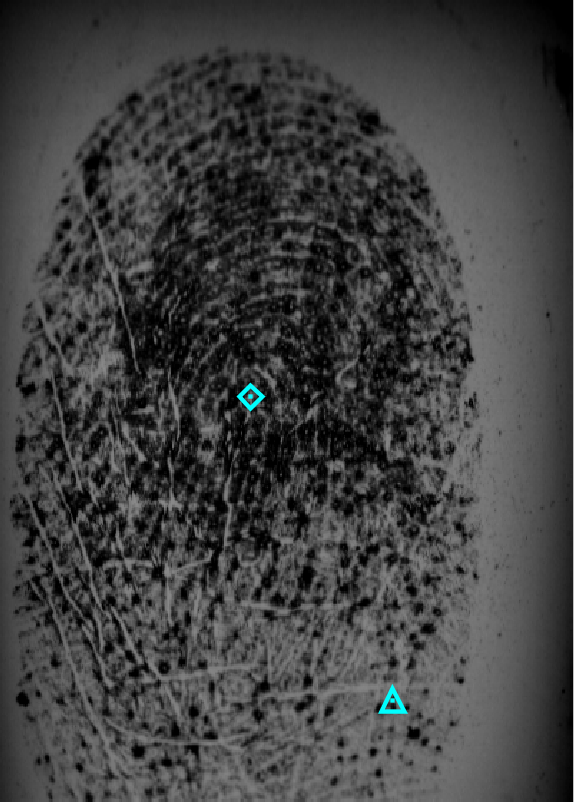}\label{fig:MarkSPs}}\hspace{0.033\textwidth}\subfigure[initial OF]{\includegraphics[width=0.45\textwidth]{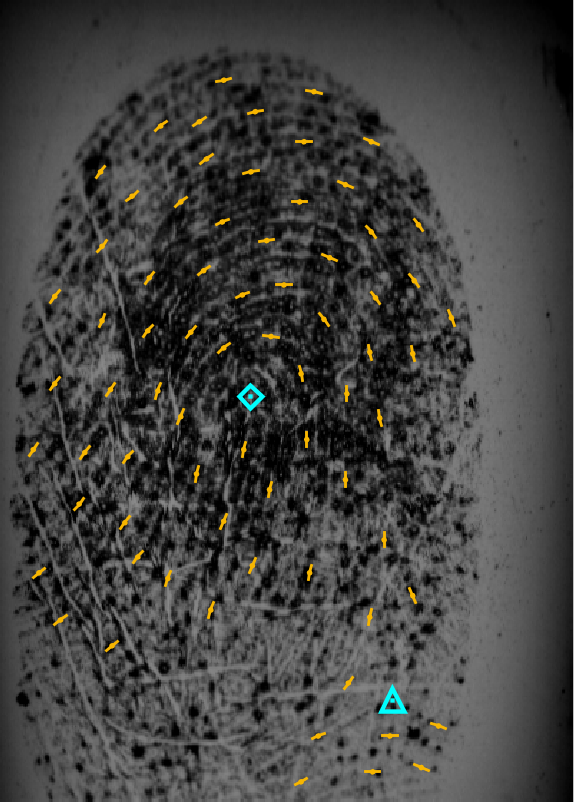}\label{fig:MarkInitialOF}}\\
\subfigure[adjust the QD]{\includegraphics[width=0.45\textwidth]{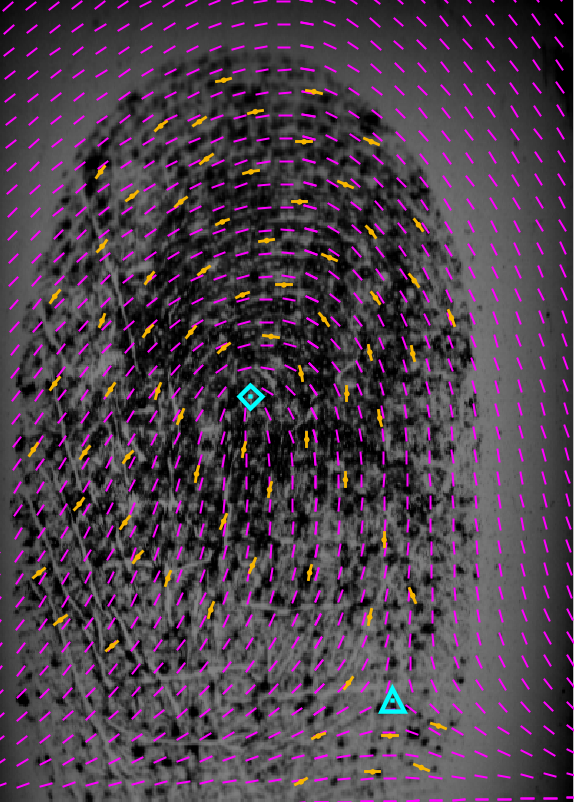}\label{fig:AdjustQDM}}\hspace{0.033\textwidth}\subfigure[insert \anglepoints{} (red)]{\includegraphics[width=0.45\textwidth]{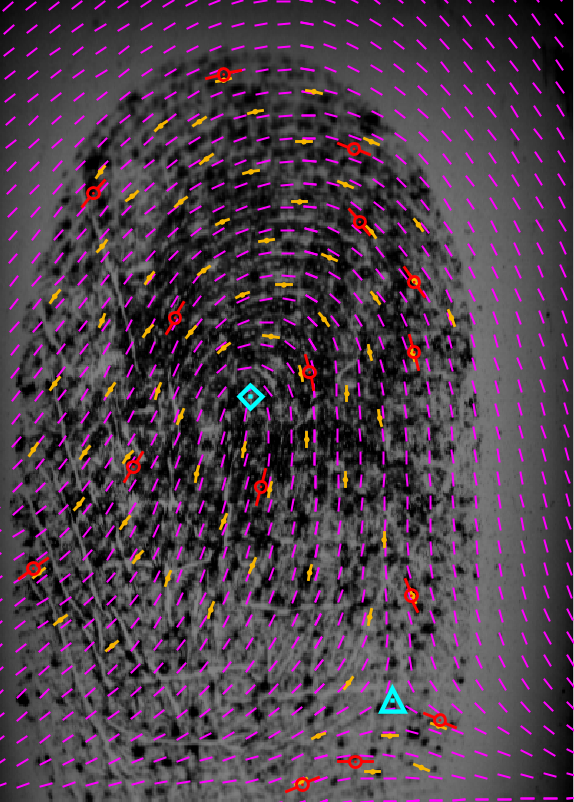}\label{fig:InsertAPs}}
\caption{Manually marking the OF of a fingerprint from FOE and modelling it by XQD.}
\label{fig:mark}
\end{center}
\end{figure}

Given the OF of an XQD model, \IE{} $A(x,y;...)$, and an initial OF $(x_j,y_j,\theta_j)$, we can measure the deviation of the XQD model to the initial OF by the following objective function, which, depends on all parameters of the XQD model:
\begin{eqnarray} \nonumber
\label{eq:ObjectiveFunction}
\kappa(\cdot)&=&\sum_{j}(~\cos(2\cdot A(x_j,y_j;...))-\cos(2\cdot\theta_j)~)^2+(~\sin(2\cdot A(x_j,y_j;...))-\sin(2\cdot\theta_j)~)^2.\\
&=&\sum_{j} \left|e^{2\sqrt{-1}A(x_j,y_j;...)} - e^{2\sqrt{-1}\theta_j}\right|^2\,.
\end{eqnarray}
We note that, if steps \ref{step:MarkSPs} and \ref{step:MarkInitialOF} have been performed manually by an expert, the remaining steps can be implemented to run (semi-)automatically by utilizing a steepest descent method applied to the objective function \eqref{eq:ObjectiveFunction}.

\subsection{Compressing Orientation Fields} \label{sec:ofcompression}

The key property of the XQD model is its ability to compactly represent, while having the power of arbitrarly well approximating, a fingerprint's OF. 
More specifically, recall that we count a total of at most $13$ real parameters describing a QD model: the parameters $\lambda$ and $R$ (see \eqref{eq:arch}) describing size and stretching, a two-fold parameter for translation, and one more parameter for rotation; further, a fingerprint can contain at most two cores and two deltas. As an XQD model is influenced by a variable number of \anglepoints{} each described by $5$ real parameters, an XQD model with $n$ \anglepoints{} consumes a total of at most
\begin{equation} \label{eq:numberParameters}
5 + 2\cdot s + 5\cdot n
\end{equation}
real parameters, where $s$ is the number of singular points.

Given an uncompressed OF, an XQD model can ideally be approximated automatically with a small number of \anglepoints{} to compress the field. In \figref{fig:mark} a manually marked fingerprint (from FOE, here assuming no ground truth OF available) has been modeled by a XQD with $n=15$ anchor points. At this point we stress that the XQD model requires a reasonable estimation of the singular points --- even if they lie outside of the fingerprint's region of interest. Unfortunately, to date there is no method known that robustly estimates all singular points. Beyond that, however, we are able to automatically obtain an XQD model from an OF thereby obtaining an effective method for compressing OFs.

\section{Perfectly Adapting Orientation Fields in the Limit %Convergence Considerations
}\label{convergence}

    It is general consent that fingerprint OFs are smooth except for the singularities at cores and deltas (e.g. \cite[Section 3.6]{HandbookFingerprintRecognition2009}). In consequence, denoting with
    $$dz_T(z) = e^{\sqrt{-1}A(x,y)}$$
    the complex orientation of the true field at pixel location $z=(x,y)$ and denoting by 
    $$dz_Q(z) = \frac{Q(z)}{|Q(z)|}$$
    the complex orientation of a QD model at $z=(x,y)$  with the same cores and deltas we may assume that there is a Lipschitz constant $L>0$ such that
    \begin{eqnarray}\label{Lipschitz:eq}
     \left| \left(\frac{dz_T(z_1)}{dz_Qz(z_1)}\right)^2- \left(\frac{dz_T(z_2)}{dz_Q(z_2)}\right)^2\right| &\leq& L |z_1-z_2| 
    \end{eqnarray}
    for all $z_1,z_2$ in the observation window,
    while of course
   $$ \frac{\left| dz_T(z_1)^2 - dz_T(z_2)^2 \right|}{|z_1-z_2|}$$
   is unbound near the singularities. 
	%Figure \ref{Lipschitz} illustrates that $L$ can be very low for QDMs well fit to a fingerprint. 
   
   For the following we assume that we have fit a QD model to a fingerprint's OF with same singularities, such that we can assume (\ref{Lipschitz:eq}) for
   $$f(z) := (dz_T(z)/dz_Q(z))^2\,.$$
   According to the algorithm introduced above, given an approximation $f_n(z)$ to $f(z)$ and a correction function $h(z;p)$, for our convergence considerations here we use not the one given by \eqref{eq:CorrectionAngle0} but a tent function 
        \begin{eqnarray}\label{tent-fcnt:eq}h(z;p,r) &=& \left(1- \frac{|z-p|}{r}\right)^+ \mbox{ where } a^+=\left\{\begin{array}{ll} a&\mbox{ for }a \geq 0\\0 &\mbox{ for }a\leq 0\end{array}\right.\end{eqnarray}
      for suitable $r>0$, we first show that the next iterate 
      $$f_{n+1}(z;p,r) := h(z;p,r) f(p) + \big(1- h(z;p,r)\big) f_n(z)$$ 
      is closer to $f$ than the previous. Building on that we then propose an algorithm, theoretically assuring an asymptotically perfect adaption to the OF.

      \begin{Lem} For fixed location $p$ and radius $r>0$, with the Lipschitz constant $L$ from (\ref{Lipschitz:eq}), the following hold:
      \begin{enumerate}
      \item[(i)] $f_{n+1}(z;p,r) = f_n(z;p,r)$ whenever $|z-p|\geq r$,
      \item[(ii)] if $\big|f(z) - f_{n}(z;p_0,r)\big|<\epsilon$ for some $\epsilon >0$ and $z$ with $|z-p|< r$ then
      $$ \big|f(z) - f_{n+1}(z;p_0,r)\big|< \left(1+\frac{\epsilon}{rL}\right)^2\frac{rL}{4}\,,$$
      \item[(iii)] choosing $r = \mathop{\sup}_z\big|f(z) - f_{n}(z;p_0,r)\big|/L$ we have 
      $$ \begin{array}{lrcll}
         (a)&  \big|f(z) - f_{n+1}(z;p_0,r)\big|&<& \mathop{\sup}_w\big|f(w) - f_{n}(w;p_0,r)\big| &\mbox{ for all }|z-p|< r\,,\\
         (b)&  \big|f(z) - f_{n+1}(z;p_0,r)\big|&<& \frac{1}{2}\,\big|f(z) - f_{n}(z;p_0,r)\big| &\mbox{ for all }|z-p|\leq 0.13\cdot r\,.
         \end{array}$$
      \end{enumerate}
      \end{Lem}
      
      \begin{proof}
      The first assertion follows from construction. For the second set  $x=|z-p_0|<r$. Then
       \begin{eqnarray} \label{proof-lem1:eq} \nonumber
       \lefteqn{
        \big|f(z) - f_{n+1}(z;p_0,r)\big|
        }\\ \nonumber
        &=&\Big|h(z;p_0,r)\big(f(z) - f(p_0)\big) + \big(1-h(z;p_0,r)\big)\big(f(z)-f_n(z)\big)\Big|\\  \nonumber
        &\leq & \left(1-\frac{x}{r}\right) x L + \frac{x}{r}\,\big|f(z)-f_n(z)\big|\\ %\label{proof-lem2:eq}
        &< &\left(1-\frac{x}{r}\right) x L + \frac{x}{r}\,\epsilon\,.
        \end{eqnarray}
        Taking the maximum of the last expression over $0\leq x\leq r$ yields (ii).
        
        (iii): With the choice for $r$, 
        setting $\epsilon = \mathop{\sup}_z\big|f(z) - f_{n}(z;p_0,r)\big|/L$, the right hand side of (\ref{proof-lem1:eq}) attains its maximum $\epsilon$ at $x=r$ and the value $\epsilon/2$ at $x=r(1-\sqrt{3}/2) \geq 0.13\cdot r$.
      \end{proof}
      
      \begin{Th}
	  With the algorithm of the following proof every fingerprint OF can be perfectly adapted in the limit, i.e.
	  $$ \sup_{z}|f_n(z) - f(z)|\to 0\mbox{ as }n\to \infty\,.$$
      \end{Th}
      
      \begin{proof}
       We detail one iteration step of the algorithm and then show its convergence. Suppose after the $(n-1)$-st iteration, $n\geq 0$, we have an approximation $f_{n-1}$ with
      $\epsilon_{n-1} =  \mathop{\sup}_z\big|f(z) - f_{n-1}(z)\big|$. Set $r=\epsilon_{n-1}/L$ and place a finite number $k(r)= O(1/r^2)$ of anchor points $p_1,\ldots,p_{k(r)}$ such that 
      $$\bigcup_{j=1}^{k(r)} B_{0.13\cdot r}(p_j)$$
      covers the fingerprint area, here $B_{\rho}(p) = \{z: |z-p|< \rho\}$. Setting $g_0:=f_{n-1}$, define  
      $$g_{j}(z) := h(z;p_{j},r) g(p_j) + \big(1- h(z;p_{j},r)\big) g_{j-1}(z),\quad j=1,\ldots,k(r)$$ and $f_n := g_{k(r)}$.
      In every step $j=1,\ldots,k(r)$, due to (iii) (b) of the above Lemma, the approximation error within $B_{0.13\cdot r}(p_j)$ is below $\epsilon_{n-1}/2$, everywhere else, due to (iii) (a), the error will still be bound by $\epsilon_{n-1}$. According to (ii), the next iteration will change the error within $B_{0.13\cdot r}(p_j)$ to below $9/16 \cdot\epsilon_{n-1}$. Since the mapping $\phi(\alpha) = (1+\alpha)^2/4$ maps the interval $(0,1)$ to itself, after at most $k(r)\leq \lfloor C/\epsilon_{n-1}^2\rfloor$ iterations, with some constant $C>0$ independent of $\epsilon_{n-1}$, we have
      $$ \sup_z\big|f_{n}(z) - f(z)\big| < g^{\lfloor C/\epsilon_{n-1}^2\rfloor}(0.5) %\left(\frac{1}{2}\right)
      \cdot \epsilon_{n-1} =: \epsilon_n < \epsilon_{n-1}\,.$$
      
      Now suppose that the sequence $\epsilon_0 > \epsilon_1 > \epsilon_2 > \ldots$ would not converge to zero but to to $\epsilon_\infty >0$. Then due to $g^{\lfloor C/\epsilon_\infty^2\rfloor}(0.5)< 1$ and pointwise monotonicity in iterates,
      $$ \epsilon_{k+1} =  g^{\lfloor C/\epsilon_k^2\rfloor}(0.5)\cdot \epsilon_k < \left(g^{\lfloor C/\epsilon_\infty^2\rfloor}(0.5)\right)^k\cdot \epsilon_0 \to 0\mbox{ as }k\to\infty\,,$$
      yielding a contradiction. This proves that every OF can be assymptotically perfectly adapted.
      \end{proof}

\section{Compression Results} \label{sec:results}

Here we report compression results using the ten good quality OFs provided by \cite{TurroniMaltoniCappelliMaio2011} as ground truths. 
As detailed in Section \ref{sec:ofcompression}, we have first manually marked singular points
and afterwards automatically fit XQD models employing the following several optimization strategies.
Stopping criteria and specific improvement steps in each iteration depend on the choice of to the main goal 
which can be:

\begin{itemize}
 \item As fast as possible (minimal runtime) in order to achieve a small deviation of the reconstructed OF from the ground truth.
 \item As exact as possible (minimal deviation from the ground truth OF) where we allow e.g. at most $20$ anchor points. 
 \item As compressed as possible (minimal file size of the stored XQD) 
 \item As sparse as possible (minimal number of anchor points)  %in order to achieve a small deviation of the reconstructed OF from the ground truth.
\end{itemize}

Note that in consequence of (\ref{eq:numberParameters}) the model's sparsity relates directly
to the compression rate: Minimizing the number of anchor points is
equivalent to a aiming for high compression, see Table \ref{table:compress} and \ref{table:compressFilesize}.
At every iteration step several choices are possible. 
One may optimize speed by simply adding a few anchor points without optimizing all possible parameters (e.g. strategy S1). 
Alternatively, when accuracy is optimized (e.g. strategy S4), in every iteration step not only all present anchor points are optimized 
but as well the choice of singular points and the other parameters of the underlying QD model are reconsidered. 
Balancing the three main goals of speed, compression rate and accuracy of the reconstructed OF allows for a range of intermediate strategies (e.g. S2 and S3).
Results for four example strategies using a grid spacing of 12 pixels for ground truth orientation locations are reported in \tableref{table:compress}.

\begin{table}[h!]
{\footnotesize
$$
\begin{array}{c|ccc|ccc|ccc|}
& \multicolumn{3}{c|}{\mbox{deviation}}& \multicolumn{3}{c|}{\mbox{runtime}}& \multicolumn{3}{c|	}{\mbox{anchor points}}\\%&\mbox{Com-}\\
& \multicolumn{3}{c|}{\mbox{(degrees)}}& \multicolumn{3}{c|}{\mbox{(seconds)}}& \multicolumn{3}{c|}{\mbox{(number)}}\\%&\mbox{pression}\\

\mbox{strategy}  & \mbox{min}&\mbox{median}&\mbox{max} & \mbox{min}&\mbox{median}&\mbox{max} & \mbox{min}&\mbox{median}&\mbox{max}\\ \hline%&\mbox{median}\\ \hline

S1& 4.3 & 4.7 & 4.9 & 1.0 & 4.4 & 8.9 & 1 & 2 & 3\\
S2& 3.5 & 4.1 & 4.8   & 1.2 & 5.8 & 23.2 & 1 & 3 & 8\\
S3 & 3.1 & 3.5 & 4.2 & 3.4 & 13.3 & 100.0 & 3 & 6 & 20\\
S4& 1.0 & 1.5 & 1.8  & 130.5 & 180.0 & 230.0   &20&20&20

\end{array}
$$
}
\caption{Employing four strategies for OF compression on the FOE's \cite{TurroniMaltoniCappelliMaio2011} good quality fingerprint images. 
Runtime has been evaluated using a single core of at a 2.8 GHz processor.}
\label{table:compress}
\end{table}

\begin{table}[!ht]
{\footnotesize
$$
\begin{array}{c|ccc|ccc|ccc|}
& \multicolumn{3}{c|}{\mbox{file size XQD}}& \multicolumn{3}{c|}{\mbox{compression factor}}& \multicolumn{3}{c|	}{\mbox{compression factor}}\\%&\mbox{Com-}\\
& \multicolumn{3}{c|}{\mbox{(bytes)}}& \multicolumn{3}{c|}{\mbox{BMP to XQD}}& \multicolumn{3}{c|}{\mbox{PNG to XQD}}\\%&\mbox{pression}\\

\mbox{strategy}  & \mbox{min}&\mbox{median}&\mbox{max} & \mbox{min}&\mbox{median}&\mbox{max} & \mbox{min}&\mbox{median}&\mbox{max}\\ \hline%&\mbox{median}\\ \hline

S1& 73  & 103 & 113 & 1975 & 2193 & 3083 & 194 & 240 & 344\\
S2& 73  & 129 & 229 & 940  & 1719 & 3083 & 113 & 195 & 277\\
S3& 113 & 181 & 453 & 497  & 1222 & 1992 & 57  & 141 & 191\\
S4& 437 & 453 & 469 & 459  & 497  & 515  & 45  & 53  & 57

\end{array}
$$
}
\caption{Employing four strategies for OF compression on the FOE's \cite{TurroniMaltoniCappelliMaio2011} good quality fingerprint images. 
For example, compressing a file of size 2000 bytes to a file of size 10 bytes would correspond to a compression factor of 200.}
\label{table:compressFilesize}
\end{table}

\section{Conclusion} \label{sec:conclusion}

In this work we have presented a semi-automatic tool based on a comprehensive XQD model for the orientation field of fingerprints, 
that achieves arbitrary precision at a very high compression rate in rather short time. 
A compression by a factor of $200$, say, at an accuracy of a few degrees in a few seconds (see \tableref{table:compress} and \ref{table:compressFilesize})

This semi-automatic tool can also be used for fast marking of orientation fields of fingerprints, 
be it for forensic application or in order to generate large orientation field benchmark databases. 
After labeling (or accepting the tool's proposals) of singular points, location, orientation and scaling of a fingerprint image, 
a very sparse representation of the orientation field with arbitrary precision is fast and automatically built. 

While in order to give a proof of concept, we have used the benchmark FOE dataset, 
in future work orientation field estimation methods (e.g. \cite{Gottschlich2012}) can be combined with our XQD model allowing for a  
'next generation comprehensive low-dimensional fingerprint template' consisting of minutiae 
plus segmentation (e.g. \cite{ThaiHuckemannGottschlich2016}) plus anchor points (XQD) 
plus at most 13 parameters (QD). One may even consider to place the anchor points at minutiae locations, 
then only their $\sigma$s need to be recorded. 
Additionally, ideal locations of anchor points -- these give the deviation from a conformal QD model -- deserve to be studied over large databases.

\section*{Acknowledgements}

The authors gratefully acknowledge the support of the 
Felix-Bernstein-Institute for Mathematical Statistics in the Biosciences 
and the Niedersachsen Vorab of the Volkswagen Foundation. 
Stephan Huckemann expresses gratitude to the support by the SAMSI Forensics Workshop 2015/16.

\end{document}